\documentclass[letterpaper, 10 pt, conference]{ieeeconf}
\IEEEoverridecommandlockouts  
\usepackage{cite}

\let\labelindent\relax 

\usepackage{amsmath,amssymb,amsfonts}
\usepackage{algorithmic}
\usepackage{graphicx}
\usepackage{textcomp}
\usepackage{xcolor}
\def\BibTeX{{\rm B\kern-.05em{\sc i\kern-.025em b}\kern-.08em
    T\kern-.1667em\lower.7ex\hbox{E}\kern-.125emX}}
\usepackage{amsmath, amsthm}
\usepackage{romannum}
\usepackage{mathrsfs}
\usepackage{float}
\usepackage{romannum}
\usepackage{hyperref}
\newtheorem*{example}{Example}
\newtheorem{ass}{Assumption}
\newtheorem{problem}{Problem}

\newtheorem{theorem}{Theorem}

\newtheorem{definition}{Definition}
\usepackage{bm}
\usepackage[ruled,vlined]{algorithm2e}
\usepackage{wrapfig}
\usepackage{subcaption}
\usepackage{multirow} 
\usepackage{enumitem}

\begin{document}

\title{\LARGE \bf Temporal-Logic-Aware Frontier-Based Exploration\\
}

\author{Azizollah Taheri and Derya Aksaray
\thanks{A. Taheri is a PhD student in the Department of Electrical and Computer Engineering at Northeastern University.}
\thanks{D. Aksaray is an Assistant Professor in the Department of Electrical and Computer Engineering at Northeastern University.}
}

\maketitle
\begin{abstract}
This paper addresses the problem of temporal logic motion planning for an autonomous robot operating in an unknown environment. The objective is to enable the robot to satisfy a syntactically co-safe Linear Temporal Logic (scLTL) specification when the exact locations of the desired labels are not known a priori. We introduce a new type of automaton state, referred to as \emph{commit states}. These states capture intermediate task progress resulting from actions whose consequences are irreversible. In other words, certain future paths to satisfaction become not feasible after taking those actions that lead to the commit states. By leveraging commit states, we propose a sound and complete frontier-based exploration algorithm that 
strategically guides the robot to make progress toward the task while preserving all possible ways of satisfying it. The efficacy of the proposed method is validated through simulations.
\end{abstract}


\section{Introduction}
Exploration is a fundamental problem that is concerned with efficiently guiding a mobile robot through an unknown environment to construct a map\cite{stachniss2009robotic, burgard2005coordinated}. A widely adopted strategy to address this problem is frontier-based exploration, where the robot selects boundaries between known and unknown regions (frontiers) as candidate targets for exploration \cite{yamauchi1997frontier}. Beyond the frontier-based exploration approach, the problem has been addressed using alternative methods, which include information-theoretic techniques (e.g., \cite{bourgault2002information}) and reinforcement learning (e.g., \cite{li2019deep}).

The case of satisfying a task while exploring the environment is more challenging, as it requires guiding the exploration 
under the task's constraints \cite{guo2023recent, zhao2023explore}.
Existing methods rely on an implicit assumption that any progress toward task satisfaction is beneficial. This is not always true. Some tasks offer multiple possible paths to satisfaction, but committing early to one option may preclude the possibility of satisfying the task through the other, which could lead to failure.

Consider a scenario where a search-and-rescue robot is deployed in a post-disaster environment with the mission of locating survivors, identifying safe exits, and determining viable exit routes. The terrain is unknown, hazardous, and non-reversible in some areas (e.g., the robot may be able to traverse downhill slopes but not climb back up due to debris, unstable structures, or slippery inclines). One strategy is to fully explore the disaster zone before initiating any rescues, which ensures complete knowledge of survivor locations and safe exits. While this helps planning with certainty, it will not be suitable for time critical missions. 
An alternative is to begin rescues as soon as sufficient local information becomes available, such as moving toward the survivor once it is observed (as in existing works). While this may potentially lead to faster mission success, premature commitments may have undesirable consequences. For example, if the robot proceeds toward a survivor via a one-way downhill path without first confirming the presence of a safe exit beyond it, it may become trapped. Therefore, the robot must distinguish between actions that cause irreversible consequences and those that carry no long-term effects. Such distinction enables the robot to balance timely progress with the preservation of future options.
There are other scenarios where completing a subtask may influence the feasibility of completing other subtasks. These observations highlight the need for more sophisticated planning strategies, rather than simply avoiding task violations and treating all progress as beneficial.

In this paper, we address the problem of satisfying a syntactically co-safe Linear Temporal Logic (scLTL) specification in an unknown discrete environment without any prior information about the location of desired regions. Our method uses a frontier-based exploration strategy. Unlike existing approaches, it computes paths to frontiers in the product space instead of the physical space. 
Computing the paths over the product space enables the robot to strategically assess which forms of progress are beneficial for task satisfaction and which may be risky or lead to negative consequences in the future. As a result, our approach offers improved performance along with guarantees for task satisfaction and safety, even in cases where state-of-the-art methods fail.

\section{Related Work}

The problem of planning in unknown environments can be broadly categorized into two groups. The first group includes cases where some prior information about the location of desired regions or objects is available (e.g.,\cite{kantaros2022perception, taheri2025motionplanningtemporallogic, taheri2026temporal}). These works leverage available information to generate an optimized plan for accomplishing the mission. The second group addresses scenarios where no prior knowledge of the location of desired regions is available (e.g., \cite{lamanna2021online,xu2022framework,ayala2013temporal,grover2021semantic}). In \cite{lamanna2021online} and \cite{xu2022framework}, the authors focus on tasks modeled using the Planning Domain Definition Language (PDDL). In \cite{lamanna2021online}, the authors propose a method in which the robot first explores the environment to gather sufficient information and then generates a plan to achieve the goal. In contrast, \cite{xu2022framework} argues that this sequential approach is inefficient, as it may cause the robot to revisit previously explored areas. They suggest that once some relevant information is discovered, the robot should begin progressing toward the goal, even if a complete path is not yet known, since this can improve the overall cost of the resulting plan. 

In \cite{ayala2013temporal} and \cite{grover2021semantic}, the authors address planning in unknown environments, where the task is specified as an scLTL formula. 
The most related work to our paper is \cite{ayala2013temporal}, which proposes a frontier-based exploration strategy for discrete environments. In their approach, the robot computes the shortest path to the frontier over the physical space and checks whether following that path would violate the task. If so, the frontier is discarded. However, this approach prioritizes immediate progress and overlook longer alternative paths that do not violate the specification. In contrast, our method computes paths over the product space, which directly incorporates task constraints and allows the robot to anticipate potential future limitations while planning. 
In \cite{grover2021semantic}, a sampling-based method is used to solve the same problem in continuous spaces. It similarly promotes incremental task progress and demonstrates improved performance compared to strategies that fully explore the environment before planning.
Nevertheless, a critical gap remains: current methods cannot simultaneously (1) make progress toward task completion during exploration and (2) assess the long-term consequences of such progress. Without this capability, premature commitments may lead to dead ends, which result in wasted resources and potential deadlocks in accomplishing the task.

\section{Preliminaries}
\subsection{Transition System}

We define a transition system (TS) to model the robot’s motion in an unknown environment as follows:  

\begin{definition}[Weighted Deterministic Transition System]
A weighted deterministic transition system (TS) is a tuple  $\mathcal{T} = (X, \Sigma, \delta, O, L,w)$, where:
\begin{itemize}
    \item $X$ is a finite set of states,
    \item $\Sigma$ is a finite set of actions,
    \item $\delta: X \times \Sigma \rightarrow X$ is a deterministic transition function,
    \item $O$ is a finite set of observations,
    \item $L: X \rightarrow 2^O$ is a labeling function that assigns a subset of observations to each state,
    \item $w: X \times \Sigma \rightarrow \mathbb{N}$ is the weight function, which assigns a positive value to each state–action pair.
\end{itemize}
\end{definition}

Given an TS $\mathcal{T}$, a finite \textbf{action sequence} is denoted by $
\bm{\sigma{\scriptstyle[0:n]}} = \sigma(0)\sigma(1)\dots\sigma(n),$
where $\sigma(i) \in \Sigma$ for all $i \in \{0,1,\dots,n\}$. A finite \textbf{trajectory} induced by $\bm{\sigma{\scriptstyle[0:n]}}$ is represented as 
$
\bm{x^{\sigma{\tiny[0:n]}}} = x(0)x(1)\dots x(n+1),
$
where $x(i+1) \in \delta(x(i),\sigma(i))$. The corresponding \textbf{word} is given by 
$
\bm{l(\bm{x^{\sigma{\tiny[0:n]}}})} = l(0)l(1)\dots l(n+1),
$
where $l(i) = L(x(i))$. The \textbf{total weight} of the trajectory $\bm{x^{\sigma{\tiny[0:n]}}}$ is defined as 
$
W(\bm{x^{\sigma{\tiny[0:n]}}}) = \sum_{i=0}^{n} w(x(i), \sigma(i)).
$

\subsection{Temporal logic}
Temporal Logic (TL) is a formal language for specifying the temporal properties of dynamical systems. Among its variants, Linear Temporal Logic (LTL) is widely used to express properties over infinite-length words 
$
l = l(0)l(1)l(2)\dots,
$
where $l(i) \in 2^O$ for all $i \geq 0$. LTL's expressive capabilities have led to its widespread use in domains like control synthesis and system verification for complex tasks (e.g., \cite{4459804}, \cite{plaku2015motion}). We concentrate on a subset of LTL called scLTL in this paper.

\begin{definition}[$\text{scLTL}_{\setminus{\emph{next}}}$]  
A formula $\phi$ in syntactically co-safe Linear Temporal Logic ($\text{scLTL}_{\setminus{\emph{next}}}$), defined on the observation set $O$, is generated recursively as follows:

\begin{equation*}
\phi = \top\mid o\mid \neg o\mid \phi_1\land\phi_2\mid\phi_1\lor\phi_2\mid\phi_1\mathcal{U}\phi_2\mid\lozenge\phi_1
\end{equation*}
where $o \in O$ denotes an observation, and $\phi$, $\phi_1$ and $\phi_2$ represent $\text{scLTL}_{\setminus\emph{next}}$ formulae. The Boolean operators used are $\top$ (true), $\neg$ (negation), $\land$ (conjunction) and $\lor$ (disjunction). The temporal operators include $\mathcal{U}$ (until) and $\lozenge$ (eventually).

\end{definition}

This language does not support the globally operator, as $\text{scLTL}_{\setminus{\emph{next}}}$ restricts negation to observations. As a result, formulae such as $\neg\lozenge\neg\phi$ are not part of $\text{scLTL}_{\setminus{\emph{next}}}$. The \emph{next} operator is omitted from the syntax (e.g., \cite{kantaros2020reactive}) because specifying objectives such as locating an object with an unknown position within $t\in \mathbb{N}$ steps could impose constraints that are too restrictive in unknown environments. For brevity, we refer to $\text{scLTL}_{\setminus{\emph{next}}}$ simply as scLTL throughout the paper.

The evaluation of the semantics of scLTL formulae is done over infinite words where each element in the word is from $2^{O}$. The language of a formula $\phi$ is denoted by $\mathscr{L}_{\phi}$, which is the set of infinite words satisfying $\phi$. Despite being interpreted over infinite words (i.e., $(2^{O})^{\omega}$\footnote{$(2^{O})^{\omega}$ denotes the set of infinite words defined over $2^O$.}), the satisfaction of scLTL formulae can always be decided in finite time. Specifically, for any infinite word $l = l(0)l(1)l(2)\dots \in \mathscr{L}_{\phi}$, there exists a finite \textbf{good} prefix $l(0)l(1)\dots l(n)$ such that any infinite extension of this prefix, $l(0)l(1)\dots l(n)l'$ with $l' \in (2^{O})^{\omega}$, also satisfies $\phi$~\cite{belta2017formal}. The set of all such good prefixes is denoted by $\mathscr{L}_{\text{pref},\phi}$. For every scLTL formula, there exists a corresponding deterministic finite state automaton (DFA) that compactly encodes all words satisfying the formula.

\begin{definition}[DFA] 
We define a deterministic finite state automaton as the tuple $\mathcal{A} = (S, s_0, 2^{O}, \delta_{a}, F_{a})$, where:
\begin{itemize}

    \item $S$ denotes a finite set of states,
    \item $s_0 \in S$ represents the initial state,
    \item $2^O$ is the input alphabet,
    \item $\delta_{a}: S \times 2^O \rightarrow S$ defines the transition function,
    \item $F_a \subseteq S$ denotes the set of accepting states.

\end{itemize}
\label{DFA}
\end{definition}

The semantics of a DFA are interpreted on finite words over the alphabet $2^{O}$. For a DFA $\mathcal{A}$, a run on a word $\bm{l} = l(0)l(1)\dots l(n)$ is a sequence of states $\bm{s} = s(0)s(1)\dots s(n+1)$ such that $s(0) = s_0$, $s(i) \in S$, and $s(i+1) = \delta_{a}(s(i), l(i))$ for each $i \geq 0$. A word $\bm{l}$ is accepted if the final state of the run is in the accepting set, i.e., $s(n+1) \in F_a$. The collection of all such words forms the language of $\mathcal{A}$, denoted by $\mathscr{L}_{\mathcal{A}}$. Every scLTL formula $\phi$ over $O$ can be translated into a DFA $\mathcal{A}_\phi$ with alphabet $2^{O}$, which accepts precisely the language of good prefixes of $\phi$ (i.e.,  $\mathscr{L}_{\mathcal{A}_\phi} = \mathscr{L}_{\text{pref},\phi}$)~\cite{belta2017formal}. While this DFA encodes accepting prefixes, an extended version can be built to also capture violations.

\begin{definition}[Total DFA]
A DFA is said to be total if, for every state $s \in S$ and every input $l \in 2^{O}$, the transition function satisfies $\delta_{a}(s, l) \neq \emptyset$~\cite{lin2015hybrid}.
\end{definition}

Given any DFA $\mathcal{A}$, it is always possible to obtain a language-equivalent total DFA by adding a trash state $s_t$ and defining $\delta_{a}(s, l) = s_t$ for every pair $(s, l)$ such that $\delta_{a}(s, l)$ is undefined. Throughout the remainder of this paper, the term DFA refers to a total DFA.

\subsection{Graph Theory}
A directed graph is represented as a tuple $G = (X, \Delta)$, where $X$ is the node set and $\Delta \subseteq X \times X$ specifies directed edges. If $(x_i, x_j) \in \Delta$, then $x_j$ is called an out-neighbor of $x_i$.  The set of all out-neighbors of $x_i$ is written as $N_{x_i}$, referred to as ``neighbors" in this paper. Moreover, $N^h_{x_i}$ denotes the set of nodes that are reachable from $x_i$ in at most $h$ hops.


\section{Problem Formulation}
We consider a mobile robot operating in an unknown discrete environment. Initially, the robot only knows its initial state and the set of labels it may encounter, but it is unaware of the exact label of each state of the environment. Its objective is to accomplish an scLTL task. The robot is assumed to have perfect localization capabilities, which means it always knows its current state. Additionally, it is equipped with sensors that allow it to reveal the labels of previously unknown states. By exploring the environment, the robot can uncover and maintain a history of the labels associated with each state. Based on this setting, we consider the following problem:
\begin{problem}
Given a mobile robot operating in an unknown discrete environment as a TS $\mathcal{T} = (X, \Sigma, \delta, O, L, w)$, where the labeling function $L$ is initially unknown, the objective is to compute a finite trajectory $\bm{x^{\sigma[0:n]}}$ that satisfies a desired scLTL specification $\phi$ over $O$, if such a trajectory exists; otherwise, determine that the specification is unsatisfiable.
\label{problem}

\end{problem}
\section{Solution Approach}

Our proposed solution to Problem \ref{problem} consists of online and offline phases.
In the offline phase, we identify states that allow progress toward task satisfaction without negative consequences. In the online phase, we use this information within a frontier-based exploration algorithm. 
This strategy ensures decision-making that accounts for the long-term effects of progress and gradually guides the robot to satisfy the task.

\subsection{Offline}
As motivated in Section \Romannum{1}, not every mission progression is necessarily beneficial. In some cases, certain progress steps could restrict the ability to satisfy the task through alternative paths. In other words, a robot may commit to a progress that limits future options, whereas avoiding that progress could have enabled the task satisfaction. In this phase, our objective is to identify such progress steps that carry these consequences. To formalize this notion, we first define the concept of a DFA with an altered initial state \cite{ayala2013temporal}:

\begin{definition}[DFA with altered initial state]
For a DFA $\mathcal{A}$ and a state $s \in S$, $\mathcal{A}(s)$ denotes a version of $\mathcal{A}$ where the initial state is set to $s_0 = s$.
\label{dfa-altered}
\end{definition}

Now we formally introduce a new type of DFA state called \emph{commit state} as follows:

    

\begin{definition}[Commit state]
    
 A non-trash and non-accepting state $s \in S \setminus(\{s_t\} \cup F_a)$ is defined as a \textbf{commit state} if $\mathscr{L}_{\mathcal{A}} \nsubseteq \mathscr{L}_{\mathcal{A}(s)}
$. We denote the set of all commit states as $S_c$.
\label{commit}
\end{definition}

Definition~\ref{commit} implies that there exists an accepting word $\bm{l} \in \mathscr{L}_{\mathcal{A}}$ such that the word $\bm{l}$ is not accepting when started from the commit state $s$, i.e., $\bm{l} \notin \mathscr{L}_{\mathcal{A}(s)}$. 

\begin{example}
Figure~\ref{fig:dfa_example} illustrates the DFA corresponding to the specification
$ \phi_0 = (\neg \textbf{b} \,\mathcal{U}\, \textbf{a}) \lor (\neg \textbf{a} \,\mathcal{U}\, \textbf{b} \land \lozenge \textbf{c})$, which means ``not observing \textbf{b} before \textbf{a} or not observing \textbf{a} before \textbf{b} and eventually observing \textbf{c}". Now, consider the word $\bm{l} = \textbf{a}$. Starting from the initial state $s_0=3$ and following $\bm{l}$, the run ends in the accepting state $s=0$ implying $\bm{l} \in \mathscr{L}_{\mathcal{A}}$. In contrast, when starting from $s =1$ and following $\bm{l}$, the run does not reach an accepting state, i.e., $\bm{l} \notin \mathscr{L}_{\mathcal{A}(1)}$. Consequently, we have $\mathscr{L}_{\mathcal{A}} \nsubseteq \mathscr{L}_{\mathcal{A}(1)}$, and $s=1$ is a commit state according to Def.~\ref{commit}. Overall, this would mean that while observing $\textbf{a}$ would suffice to satisfy the task, making mission progress towards $s=1$ eliminates this possibility.
\end{example}
 

\vspace{-10mm}
\begin{figure}[!htbp]
    \centering   \includegraphics[width=0.37\textwidth]{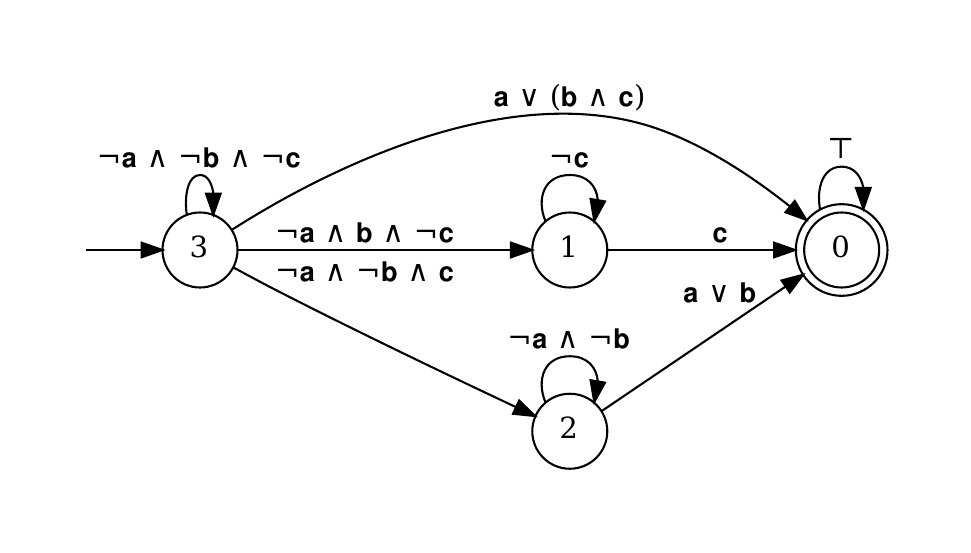}
    \vspace{-8mm}
    \caption{\scriptsize The DFA $\mathcal{A}$ corresponding to formula $\phi_0$. }
    \label{fig:dfa_example}
\end{figure}


To determine whether there exists a word that $\mathcal{A}$ accepts while $\mathcal{A}(s)$ does not, one must evaluate every word accepted by $\mathcal{A}$ on $\mathcal{A}(s)$ to verify acceptance. To address this, we propose to create a special product automaton obtained by computing the product of the DFA $\mathcal{A}$ with itself. This automaton allows the behaviors of $\mathcal{A}$ and $\mathcal{A}(s)$ to be analyzed in parallel over all words accepted by $\mathcal{A}$. By constructing this special product automaton, the problem of exhaustively checking all possible words (brute force) is transformed into a reachability analysis problem.

\begin{definition}[Product Automaton (DFA-DFA)]
    
Given a DFA $\mathcal{A}$ constructed from the specification $\phi$, we define the \textit{product automaton} as a tuple
$\bar{\mathcal{A}} = \mathcal{A} \times \mathcal{A}= (\bar{S}, \bar{S}_0,2^O,\bar{\delta},\mathcal{G})$, where:
\begin{itemize}

    \item $\bar{S} = S \times S$ is the state space,
    \item $\bar{S}_0 = s_0 \times (S \setminus(\{s_t\} \cup F_a))$ is the set of initial states,
    \item $2^O$ is the alphabet,
    \item $\bar{\delta}: \bar{S}\times 2^{O}\rightarrow\;\bar{S}$ is the transition function such that 
    $\bar{\delta}\bigl((s,s'),l\bigr) = \bigl(\delta_{a}(s,l),\,\delta_{a}(s',l)\bigr)$ where $s, s' \in S$.
    \item $\mathcal{G} = F_a \times (S \setminus F_a$) is the target set.
\end{itemize}
\label{commit_product}
\end{definition}


    Consider the sequence $\bm{\bar{s}} = \bar{s}(0)\dots \bar{s}(n+1)$ with $\bar{s}(0)\in \bar{S}_0 $ and $\bar{s}(n+1)\in \mathcal{G}$, which is the run of $\bar{\mathcal{A}}$ over a word $\bm{l} = l(0)\dots l(n)$.
    Since each $\bar{s} = (s,s') \in \bar{S}$ consists of two components, with $s, s' \in S$, the run can be viewed as two parallel sequences of DFA states: the sequence of first components $s(0) \dots s(n+1)$ with $s(0)= s_0$ and $s(n+1)\in F_a$, and the sequence of second components $s'(0) \dots s'(n+1)$ with $s'(0) = s$ and $s'(n+1)\in S \setminus F_a$. This implies that $\bm{l} \in \mathscr{L}_{\mathcal{A}}$ while $\bm{l} \notin \mathscr{L}_{\mathcal{A}(s)}$. Hence, by Def.~\ref{commit}, state $s\in S$ is classified as a commit state.
To determine the set of commit states, we proceed as follows. For each initial state $(s_{0}, s) \in \bar{S}_0$, we examine whether there exists a path in $\bar{\mathcal{A}}$ leading to some state in $\mathcal{G}$. If such a path exists, then $s$ is identified as a commit state, i.e., $s \in S_c$. 

\begin{example}
(cont.)  Consider the product automaton $\bar{\mathcal{A}} = \mathcal{A} \times \mathcal{A}$ constructed from the DFA $\mathcal{A}$ in Fig.~\ref{fig:dfa_example}. We will illustrate two portions of $\bar{\mathcal{A}}$ in Figs.~\ref{fig:product_example_commit} and \ref{fig:product_example_regular}, which are restricted to the states reachable from $\bar{s}_0 = (3,1)$ and $\bar{s}_0' = (3,2)$, respectively. Note that $\mathcal{G}=\{(0,1), (0,2), (0,3)\}$ according to Def.~\ref{commit_product}. In Fig.~\ref{fig:product_example_commit}, a state $\bar{s} = (0,1) \in \mathcal{G}$ is reachable from the initial state $\bar{s}_0 = (3,1)$, indicating that state $s = 1$ is a commit state as discussed earlier. In Fig.~\ref{fig:product_example_regular},  none of the states in $\mathcal{G}$ are reachable from the initial state $\bar{s}_0' = (3,2)$. This demonstrates that state $s = 2$ is not a commit state.

\end{example}
 
\vspace{-6mm}
\begin{figure}[!htbp]
    \centering
    \includegraphics[width=0.40\textwidth]{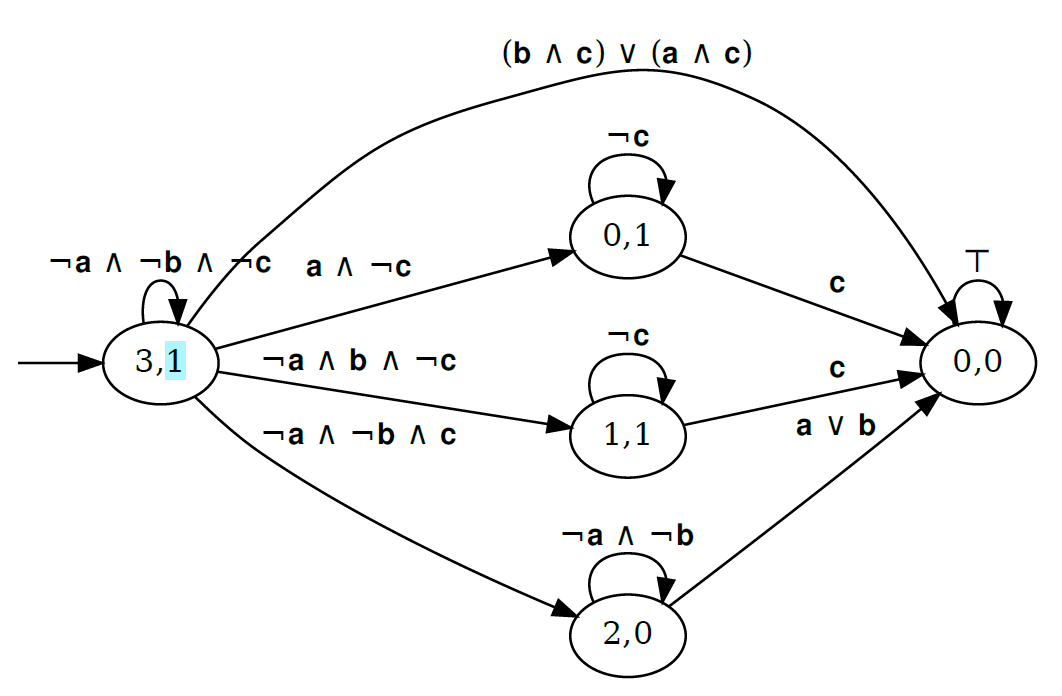}
    \vspace{-2mm}
    \caption{\scriptsize A portion of $\bar{\mathcal{A}}$ showing all states reachable from the initial state $\bar{s}_0 = (3,1)$.}
    \label{fig:product_example_commit}
\end{figure}

\vspace{-6mm}
\begin{figure}[!htbp]
    \centering
    \includegraphics[width=0.37\textwidth]{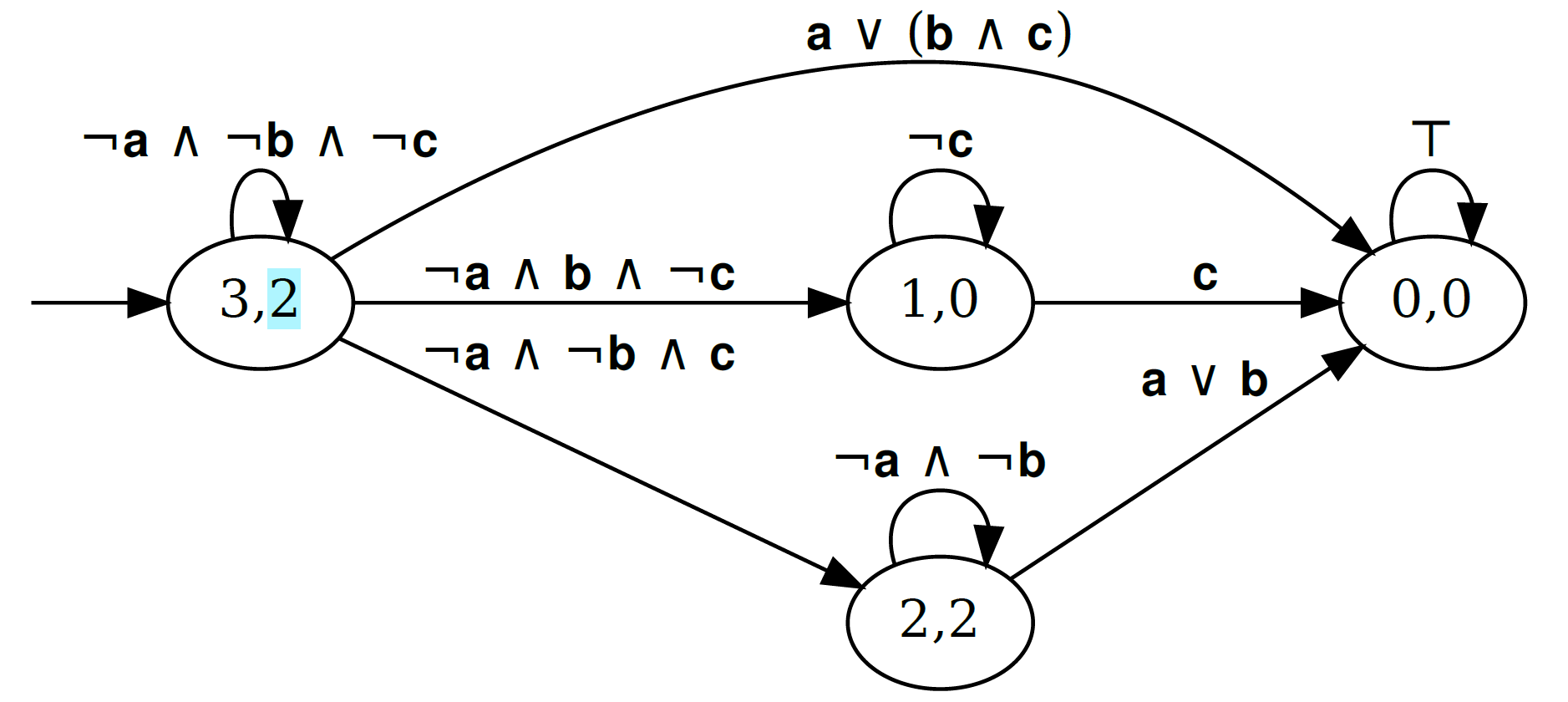}
    \vspace{-2mm}
    \caption{\scriptsize A portion of $\bar{\mathcal{A}}$ showing all states reachable from the initial state $\bar{s}'_0 = (3,2)$.}
    \label{fig:product_example_regular}
\end{figure}
\vspace{-4mm}

\subsection{Online}
We construct a product automaton from the TS $\mathcal{T}$ and the DFA $\mathcal{A}$ to jointly capture both the physical state evolution and the task progress. This product automaton forms the basis of our planning approach. Unlike the traditional product automaton construction, which is performed in a single step, our framework constructs the product automaton incrementally by using the set of known TS states.

\begin{ass}
Suppose that the robot is located in state $x \in X$. The robot's sensors, with sensing range $h \geq 1$, provide the labels $L(x)$ as well as $L(x')$ for all $x' \in N_{x}^h$. 
\label{sensor}
\end{ass}

Assumption \ref{sensor} is a mild assumption, as modern sensing technologies typically provide precise environmental measurements, which enables the robot to accurately determine state labels within its sensing range. The set of all TS states whose labels have been revealed under Assumption~\ref{sensor} is referred to as the set of \textbf{known states} and is denoted by $X_k\subseteq X$.


\begin{definition}[Product Automaton (TS-DFA)]
    Given a TS $\mathcal{T} = (X_k, \Sigma, \delta, O, L, w)$ and a DFA $\mathcal{A} = (S, s_0, 2^O, \delta_{a}, F_a, \{s_t\})$, we define their product automaton as the tuple $\mathcal{P} = \mathcal{T} \times \mathcal{A} = (S_p, S_{p_0}, \Sigma, \delta_p, p_p, \mathcal{F}_a, \mathcal{F}_t)$, where:

    \begin{itemize}
    \item $S_p = X_k \times S$ denotes the set of product states,
    \item $S_{p_0} = X_k \times \{s_0\} \subseteq S_p$ represents the set of initial states,
    \item $\Sigma$ is the finite set of actions,
    \item $\delta_p : S_p \times \Sigma \rightarrow S_p$ is the transition function defined as follows: for any $(x, s) \in S_p$ and $\sigma \in \Sigma$, $\delta_p((x, s), \sigma) = (x', s')$, where $\delta(x, \sigma) = x'$, $l = L(x')$, and $\delta_{a}(s, l) = s'$.

    \item $\mathcal{F}_a = X_k \times F_a$ is the set of accepting states,
    \item $\mathcal{F}_t = X_k \times \{s_t\}$ is the set of trash states.
    
\end{itemize}
\label{Pa}

\end{definition}

Reaching an accepting state over the product automaton ensures the satisfaction of $\phi$, whereas entering a trash state indicates a violation. Our objective is to compute a trajectory in $\mathcal{P}$ that enables the robot to reach an accepting state. However, achieving this requires knowledge of the labeling function $L$, which is unknown a priori as stated in Problem~\ref{problem}.

As the robot explores the environment, new labels are gradually revealed according to Assumption~\ref{sensor}. Consequently, additional states are added to the set of known states $X_k\subseteq X$, which leads to the incremental expansion of the product automaton with new states and edges determined by the labeling function $L$ and the transition function $\delta_p$.

Throughout the mission, at each time step, the robot evaluates reachability to an accepting state over $\mathcal{P}$. If no such state is reachable, the robot employs a frontier-based exploration strategy that prioritizes information gain while promoting task progress toward satisfaction without adverse consequences whenever possible, and consistently avoiding task violations.

\subsubsection{TL-Aware Frontier-Based Exploration}

In our approach, frontier-based exploration is executed when no accepting state in the product automaton $\mathcal{P}$ is reachable.
We define a frontier state as a known TS state, i.e., $x \in X_k$, that is adjacent to at least one unknown state $x' \in X \setminus X_k$.
To reach a frontier state, the robot computes its trajectory using the product automaton $\mathcal{P}$. A finite trajectory over $\mathcal{P}$ generated by an action sequence $\bm{\sigma{\scriptstyle[0:f-1]}}$ is represented as a sequence $\bm{s_p^{\sigma{\tiny[0:f-1]}}} = s_p(0)s_p(2)\dots s_p(f)$, where $\delta_p(s_p(i),\sigma(i)) = s_p(i+1)$ for all $i \in \{0,1,\dots,f-1\}$. For simplicity, we use $\bm{s_p}$ to denote $\bm{s_p^{\sigma{\tiny[0:f-1]}}}$ throughout the paper, except in cases where clarification is required. We denote the $i$-th element of the sequence $\bm{s_p}$ by $\bm{s_p}[i]$ and the final element by $\bm{s_p}[\text{f}]$, where $\bm{s_p}[i] = (x(i), s(i))$ and $\bm{s_p}[\text{f}] = (x(\text{f}), s(\text{f}))$. We define the total weight of a product automaton trajectory $\bm{s_p}$ as  $ W_p(\bm{s_p}) = \sum_{i=0}^{f-1} w(x(i), \sigma(i)).
$


 In order to reach a frontier $x \in X_k$, the robot must reach a product automaton state $(x, s)$ that shares the same TS state $x$. Since the product automaton may contain multiple states corresponding to the same TS state $x$ but with different DFA states, the robot can approach the same frontier via different trajectories. Each trajectory may have different implications for the task, such as no progress, partial progress, task violation, or task satisfaction, depending on the associated DFA state. By planning through the product automaton, the robot gains the ability to evaluate alternative ways of reaching a frontier while considering the impact on task progression. This approach contrasts with \cite{ayala2013temporal}, where the method selects the shortest path in the TS to a frontier and disregards that frontier if the path results in a task violation.

The objective of our frontier-based exploration is to select, among all frontier states, the one that maximizes information gain, minimizes travel distance, and promotes progress toward task satisfaction without introducing adverse consequences.

\paragraph{Information gain}
When the robot reaches a frontier, it uncovers the labels of certain states that were previously unknown. We define the information gain associated with a frontier state $x$, denoted by $\mathcal{I}(x)$, as the number of states whose labels are revealed upon visiting $x$, given the robot's sensing range $h$. Formally,  
\begin{equation}
\mathcal{I}(x) = \left| \mathcal{N}^{h}_{x} \cap \bigl(X \setminus X_k\bigr) \right|.
\end{equation}

\paragraph{Task progress metric}

 We introduce a task progress metric $\Omega(\bm{s_p})$ that is computed as:  

\begin{equation}\label{eq:reward}
    \Omega(\bm{s_p})
    =
    \small
    \begin{Bmatrix}
   -\infty & \text{      if } s(\text{f}) = s_t  \\
     -\frac{\alpha_1\cdot|X|}{\alpha_2} & \text{      if } s(\text{f}) \in S_c   \\
     \Delta_\phi( s(0),s(f) ) & \text{otherwise}
    \end{Bmatrix},
\end{equation}
where $s(\text{f})$ denotes the final DFA state in the trajectory $\bm{s_p}$, $s_t$ is the trash state, and $S_c$ is the set of commit states. The parameters $\alpha_1,\alpha_2>0$ are user-defined weighting factors that indicate the importance of information gain and task progress metric, respectively. Here, $\Delta_\phi: S \times S \rightarrow \mathbb{Z}$ is a function that quantifies progress toward task satisfaction by comparing the initial and final DFA states. In the literature, this computation is performed by finding the shortest path (in terms of hops) on a pruned version of the  DFA. In this pruned DFA, transitions requiring more than one observation to hold simultaneously are removed \cite{kantaros2022perception}. This function computes how much closer the final state is to an accepting state. A higher value of $\Delta_\phi$ indicates that the trajectory brings the robot closer to an accepting state.

By leveraging these criteria, we introduce an overall value that is assigned to each frontier: 
\begin{equation}
    V(x) = \max_{\bm{s_p}}\frac{\alpha_1 \cdot \mathcal{I}(x) + \alpha_2 \cdot \Omega(\bm{s_p}) }{W_p(\bm{s_p})^{\alpha_3}},
    \label{frontier-value}
\end{equation}
where $\mathcal{I}(x)$ denotes the information gain associated with frontier $x$, $\Omega(\bm{s_p})$ represents the task progress metric along the product automaton trajectory $\bm{s_p}$, and $W_p(\bm{s_p})$ is the corresponding total weight of that trajectory. The parameters $\alpha_1, \alpha_2,$ and $\alpha_3$ are positive weighting factors that balance the contribution of each term. The value function is designed to ensure that any trajectory leading to task violation is never selected ($\Omega(\bm{s_p})=-\infty$), thereby guaranteeing safety. Furthermore, if all trajectories to a given frontier lead the robot into a commit state, the resulting value for that frontier becomes negative, since $|X| > \mathcal{I}(x)$ for all $x \in X$. In contrast, other frontiers retain positive values. Consequently, this formulation forces the robot to avoid entering a commit state whenever alternative exploration options exist, while promoting task progress that has no consequences. An outline of the proposed method is illustrated in Fig.~\ref{fig:method}.

\vspace{-4mm}
\begin{figure}[!htbp]

  \centering
  \includegraphics[width=0.41\textwidth]
  {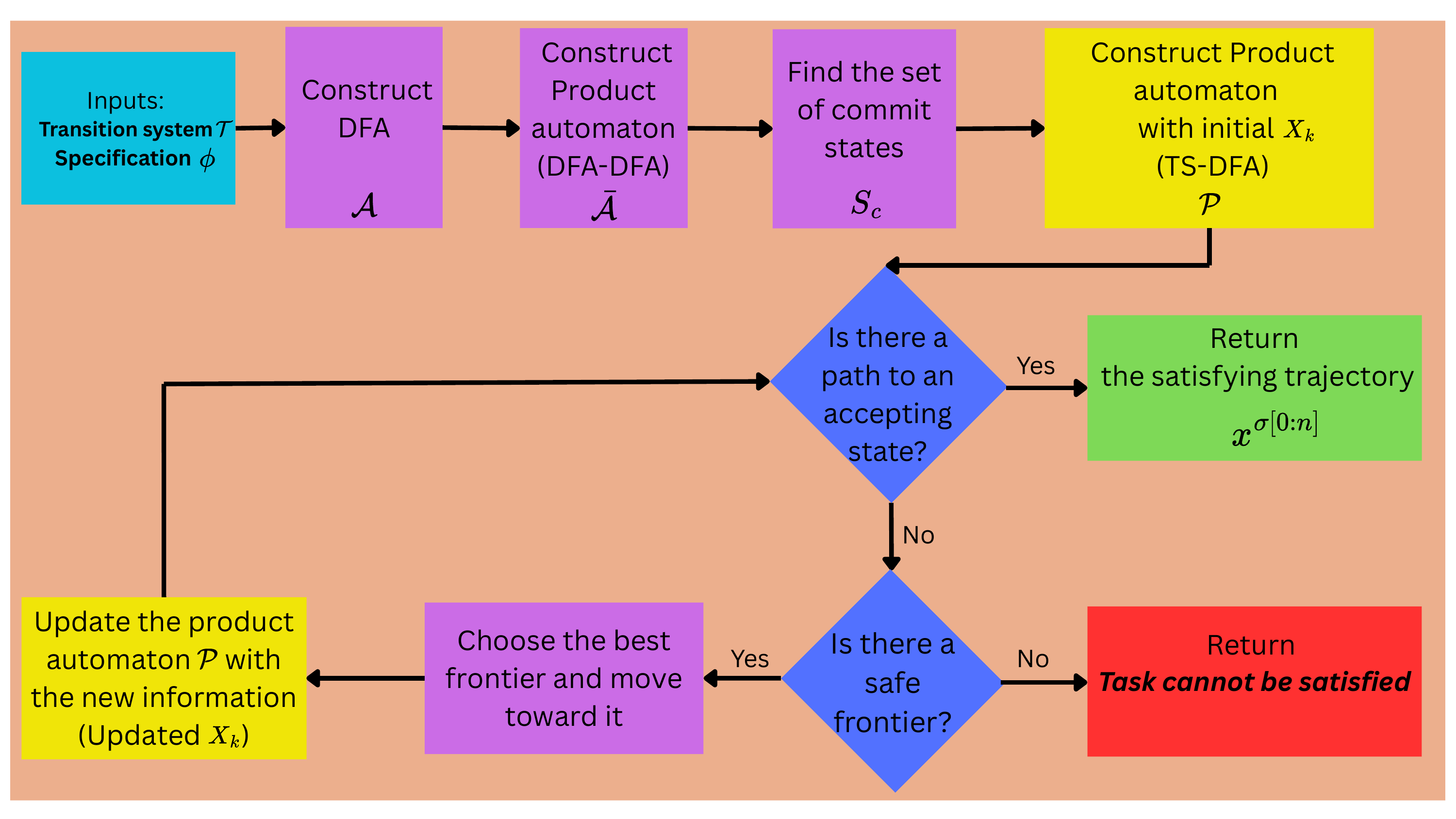}
  \caption{\scriptsize Summary of the proposed method.}
  \label{fig:method}
\end{figure}

\vspace{-2mm}
\subsubsection{Algorithm}
\LinesNumbered
\begin{algorithm}[!b]
\scriptsize

\SetKwInput{Input}{Input}
\SetKwInput{Output}{Output}

\caption{TL-Aware Frontier-Based Exploration}
\label{alg}

\Input{$\mathcal{M} = (X, \Sigma, \delta, O, L, c)$, \text{initial state }$x_0$, \text{set of commit states} $S_c$, $X_k$, $\mathcal{A} = (S, s_0, 2^O, \delta_{a}, F_a, \{s_t\})$,  }
\Output{A trajectory $\bm{x^{\sigma[0:n]}}$ satisfying $\phi$ }

\text{Construct $\bar{\mathcal{A}}= \mathcal{A} \times \mathcal{A}$ and calculate $S_c$ as per Def. \ref{commit}}

\text{Construct $\mathcal{P}=(S_p, S_{p_0}, \Sigma, \delta_p, p_p, \mathcal{F}_a, \mathcal{F}_t)$ as per Def. \ref{Pa} }

$\bm{x^{\sigma[0:n]}} = [\ ] $, ${s_{p}}_\text{current} =(x_0,s_0)$

\While{$\mathcal{F}_a$ is not reachable from ${s_{p}}_\text{current}$}{

Set $\mathscr{F}$ as the current set of all frontiers

\If{$\mathscr{F}=\emptyset$}{ \Return{Task cannot be satisfied }}

$V^*= -\infty$ and $\bm{s_p}^* = [\ ] $

 \ForEach{ $x\in \mathscr{F}$}{
         calculate $V(x)$ and $\bm{s_p}$ as per \eqref{frontier-value}
         
         \If{ $V(x)>V^*$ }{
         $V^* = V(x)$ and  $\bm{s_p}^* =\bm{s_p}$

        }

}

$V_{max}=V^*$

\If{$V_{max} = -\infty$}{\Return{Task cannot be satisfied }}

Follow $\bm{s_p}^*$, updating $X_k$, $L$, $\mathcal{P}$, $\bm{x}^{\sigma[0:n]}$ and ${s_{p}}_\text{current}$

}

Compute path to $\mathcal{F}_a$ minimizing total weight, updating $\bm{x}^{\sigma[0:n]}$




    
            
    
        
\Return{ $\bm{x^{\sigma[0:n]}}$ }
\end{algorithm}

We present our algorithm for solving Problem \ref{problem}. In line~1, the set of commit states $S_c$ is computed based on the specification $\phi$. Using the set of known states, the initial product automaton is then constructed. Lines~4--16 implement the frontier-based exploration procedure, which is triggered whenever no path to an accepting state exists in the current product automaton. Specifically, the set of all frontiers is first collected in $\mathscr{F}$, and the frontier with the highest value is selected. If no frontiers remain, or if all remaining frontiers lead to task violation, the algorithm concludes that the task is unsatisfiable. Otherwise, the robot follows the trajectory to the selected frontier and updates the product automaton with the newly discovered information. This process continues until a path to an accepting state is found, at which point the algorithm returns the traversed sequence of states $\bm{x^{\sigma[0:n]}}$.

\begin{ass}
If $-\infty<V_{max} < 0$ at anytime during the execution of Alg.~\ref{alg}, then from every state $s_p \in S_p \setminus \mathcal{F}_t$ there exists a path to some accepting state in $\mathcal{F}_p$.

     \label{entering_commit}

\end{ass}


Assumption~\ref{entering_commit} excludes scenarios where a satisfying path exists but can only be found by taking a risk. In other words, if the robot has explored part of the map without finding an accepting path, and the only remaining way to explore the rest of the map is by entering a commit state ($-\infty<V_{max}<0$), then entering that commit state will not eliminate the reachability to the set of accepting states.

Next, we establish through the following theorem that Alg.~\ref{alg} satisfies the properties of soundness and completeness.
\begin{theorem}
Let Assumptions \ref{sensor} and \ref{entering_commit} hold. Then, Alg.~\ref{alg} is sound and complete.
\label{theo}
\end{theorem}

\begin{proof}

 At any time during the execution of Alg.~1, one of the following two cases occurs:

\begin{enumerate}
    \item From the current state ${s_{p}}_\text{current}$, no accepting state is reachable in $\mathcal{P}$ (while-loop, line~4).

    In this case, there are two sub-cases regarding the existence of frontiers:
    \begin{enumerate}[label=(\alph*)]
        \item[1.a)] $\mathscr{F} = \emptyset$: The robot has explored the entire environment but has not found a satisfying path. From line~7, the algorithm terminates and returns that the task cannot be satisfied.
        
        \item[1.b)] $\mathscr{F} \neq \emptyset$: Three further subcases arise after line~12:
        \begin{enumerate}[label=(\roman*)]
            \item  $V_{max} \geq 0$: The robot selects a frontier that does not lead to a commit state. As the robot moves towards that frontier, Assumption~\ref{sensor} ensures that the size of the known state set $X_k$ increases (line~16).
            
            \item  $-\infty<V_{max}<0$: This means that the robot has already explored part of the environment, and all remaining frontiers lead to commit states. In this case, the size of the known state set $X_k$ also increases (line~16). Moreover, by Assumption~\ref{entering_commit}, a satisfying path to an accepting state exists from every state $s_p \in S_p \setminus \mathcal{F}_t$. Hence, entering a commit state under these conditions does not eliminate reachability to the set of accepting states.
            
            \item  $V_{max} = -\infty$: No safe frontier is reachable, and from line~15 the algorithm terminates, which returns that the task cannot be satisfied.
        \end{enumerate}
    \end{enumerate}
    
    Overall, in subcases (i) and (ii), where $\mathscr{F} \neq \emptyset$, the size of the known state set $X_k$ increases, while preserving the reachability to the set of accepting states. Since $\mathcal{T}$ contains a finite number of states, eventually $X_k = X$, which implies $\mathscr{F} = \emptyset$, unless some frontiers are unreachable due to potential task violations. In that event, $V_{max} = -\infty$, which leads to termination of the algorithm as in subcase (iii). Thus, in Case~1 the robot either explores a safe frontier or terminates when no safe frontiers are left, by returning that the task cannot be satisfied.
    
    \item From ${s_{p}}_\text{current}$, there exists a path to an accepting state in $\mathcal{P}$.

    This is the only case where the algorithm returns a trajectory. The trajectory produced at line~18 leads the robot to an accepting state in $\mathcal{P}$, which directly implies satisfaction of the scLTL specification $\phi$.
\end{enumerate}

\noindent In conclusion, if the algorithm outputs a trajectory, that trajectory satisfies the specification (\emph{soundness}). Furthermore, if a satisfying path exists, the robot will eventually find it by exhaustively visiting all safe frontiers (\emph{completeness}).
\end{proof}
    
\section{Case Studies}
\noindent \textbf{Simulation:} We consider a robot operating in a $20 \times 20$ grid environment, equipped with an action set 
$\Sigma = \{\text{Up},\text{Down}, \text{Right},  \text{Left}, \text{Stay}\}$ and a sensing range of \(h = 3\). 
The scLTL formula \(\phi\) is translated into its corresponding DFA using the Spot library~\cite{duret2016spot}. 
The parameters in~\eqref{frontier-value} are set to $\alpha_1 = 1$, $\alpha_2 = 20$, and $\alpha_3 = 1$. 
All computations are performed on an Intel Core i7 laptop (2.3 GHz, 16 GB RAM).

\subsection{Search and Rescue}
\begin{figure*}[t]
  \centering
  \includegraphics[width=\textwidth, trim=0 14cm 0 0, clip]{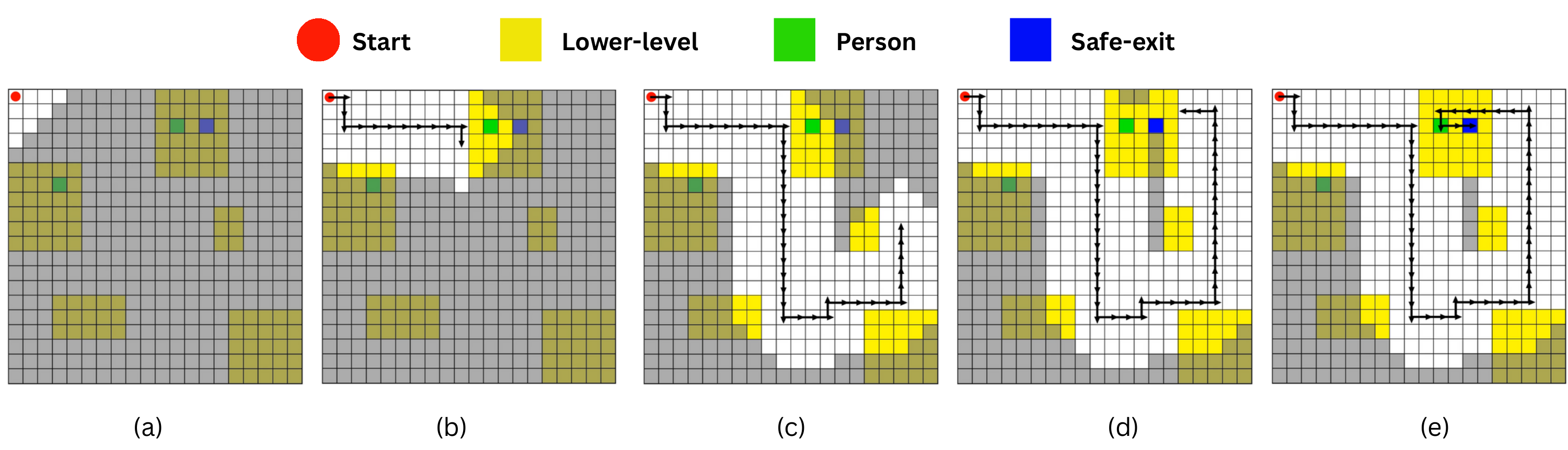}
  \caption{\scriptsize Snapshots from a video, starting at (a) and ending at (e), illustrating the robot executing our algorithm to satisfy the task $\phi_1$. The robot’s starting position is marked by a red circle in the top-left corner, while the lower-level region, person, and safe exit are represented by yellow, green, and blue squares, respectively. The robot’s traversed trajectory is shown by black arrows.}
  \label{fig:case1}
\end{figure*}

In this scenario, the robot is tasked with locating both a person and a safe exit within the environment, and subsequently rescuing the person by transporting them to the safe exit. In the environment, some regions may have slippery inclines. Once the robot enters such a region, it may descend to a lower level and be unable to return, which restricts its movement to that level. These regions are referred to as lower-level regions. This mission is specified by the scLTL formula: 
\begin{equation}
\phi_1 = (\neg \textbf{L}~ \mathcal{U} ~(\textbf{L} ~\mathcal{U}~ (\textbf{P}~ \mathcal{U}~ ( ~(\textbf{L}\lor \textbf{P})~ \mathcal{U}~ \textbf{S})) ))     
\land\lozenge \textbf{S} \land (\neg \textbf{S}~ \mathcal{U}~ \textbf{P}),
\end{equation}
where $\mathbf{P}$ denotes the person, $\mathbf{S}$ the safe exit, and $\mathbf{L}$ the lower-level region.

As in Fig.~\ref{fig:case1}, the robot begins its mission from the initial cell, marked with a red circle. Due to its limited sensing range, the robot can only observe the true labels of nearby cells, while cells with unknown labels are depicted in shaded grey. The robot proceeds to explore the environment in order to gather the information required to satisfy the specified task. In Fig.~\ref{fig:case1}(b), although the objective is to rescue the person and transport them to a safe exit, the robot does not immediately approach the person upon detection. This is because the person is located in a lower-level region, which is a one-way area from which the robot cannot return if no safe exit is available. Entering such a region is considered as a commit state in our planning algorithm; the robot avoids committing to that progress unless it can identify a satisfying trajectory beyond that point. Consequently, it prioritizes exploring other areas to seek alternative solutions. In Fig.~\ref{fig:case1}(d), the robot discovers that a safe exit exists within the lower-level region containing the person, thereby identifying a satisfying trajectory in the product automaton. Finally, as in Fig.~\ref{fig:case1}(e), the robot executes this trajectory and successfully completes the task. Note that our algorithm can be re-executed after each termination until every person, who can be saved, has been rescued.



\begin{figure}[!htbp]
    \centering
    \includegraphics[width=0.17\textwidth]{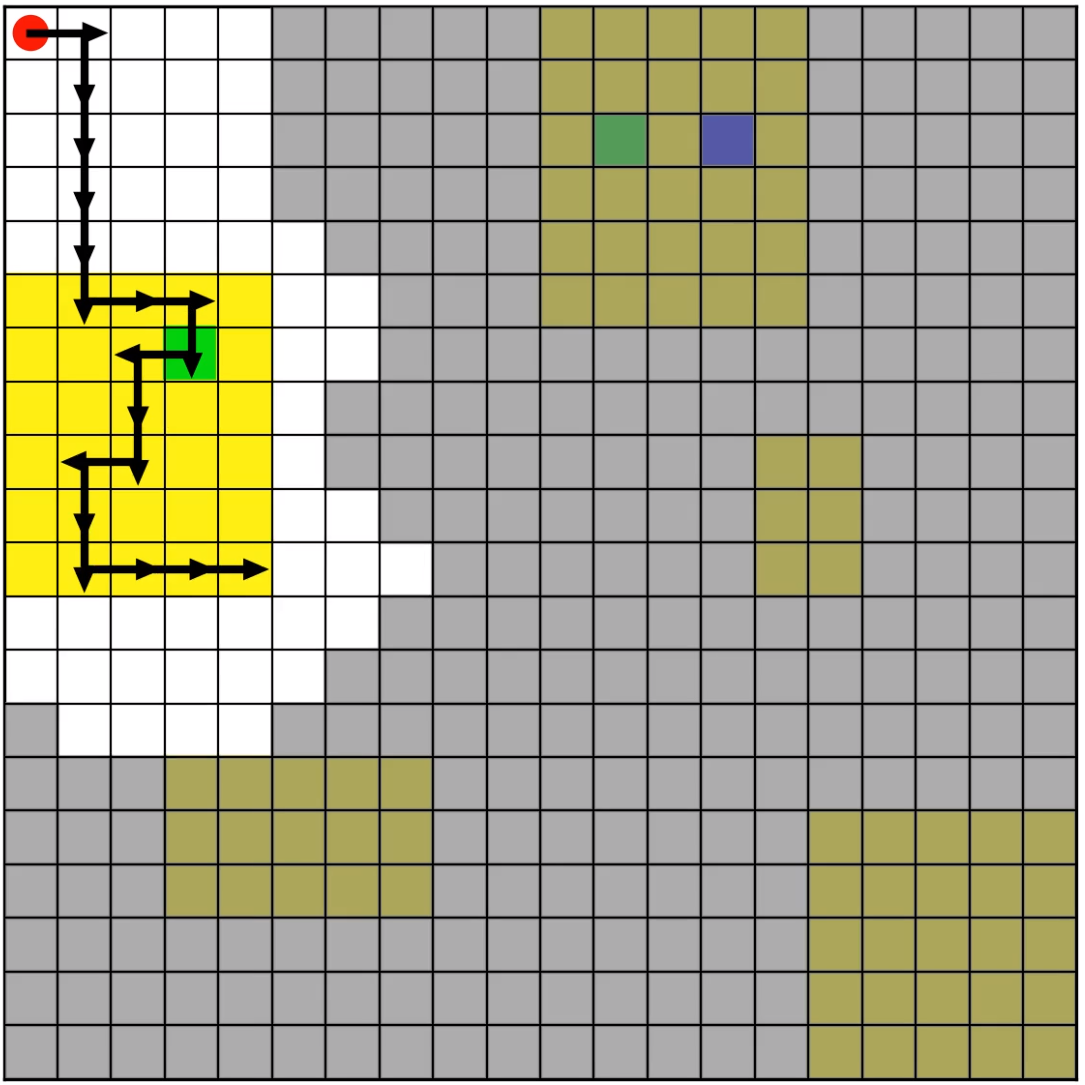}
    \caption{\scriptsize Trajectory traversed by the robot when following the method in \cite{ayala2013temporal}, which shows that it gets stuck and fails to satisfy the task.}
    \label{fig:case1_2}
\end{figure}

Figure~\ref{fig:case1_2} depicts the same scenario and map as in Fig.~\ref{fig:case1}. But, it illustrates the outcome when the planning approach \cite{ayala2013temporal} is employed. In this case, there is no mechanism to prevent the robot from entering the lower-level region. Once inside, it detects the person but the robot becomes trapped as no safe exit is available.
 Although the specification is not violated at this point, it becomes impossible to complete, which renders the mission unsatisfiable.
This highlights the importance of incorporating commit states into the planning process.



\begin{figure}[!b]
    \centering
    \begin{tabular}{ccc} 
        \begin{subfigure}[t]{0.28\columnwidth}
            \centering
            \includegraphics[width=\textwidth]{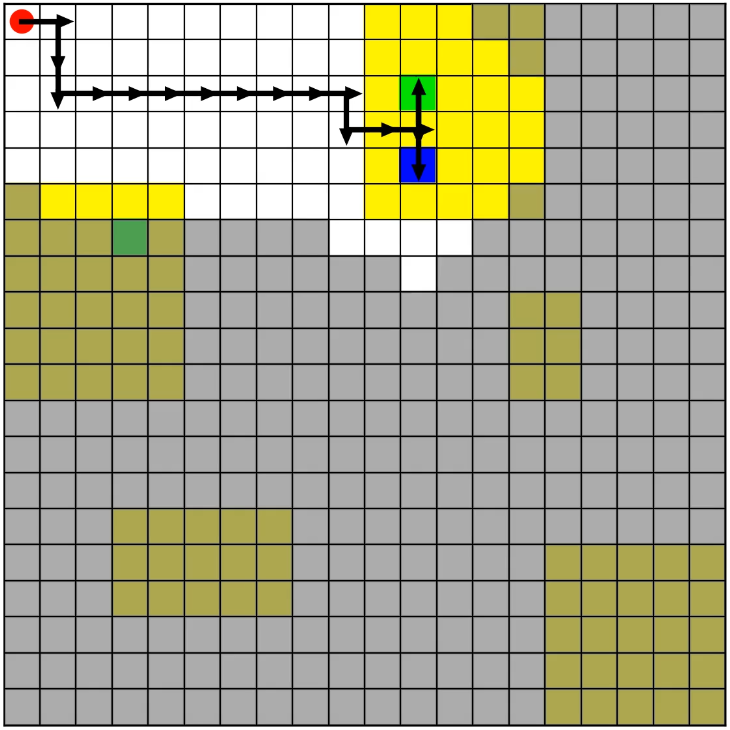}
            \caption{}
        \end{subfigure} &
        \begin{subfigure}[t]{0.28\columnwidth}
            \centering
            \includegraphics[width=\textwidth]{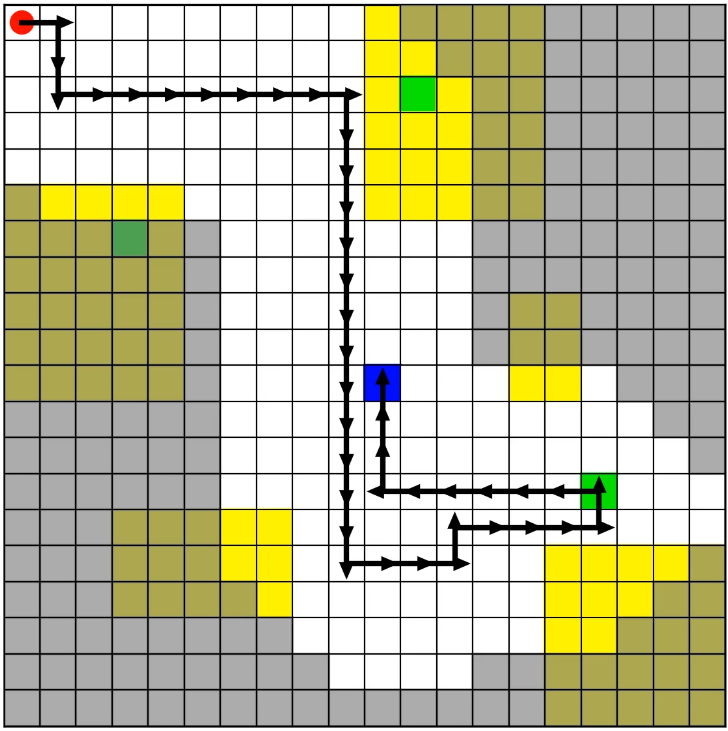}
            \caption{}
        \end{subfigure} &
        \begin{subfigure}[t]{0.28\columnwidth}
            \centering
            \includegraphics[width=\textwidth]{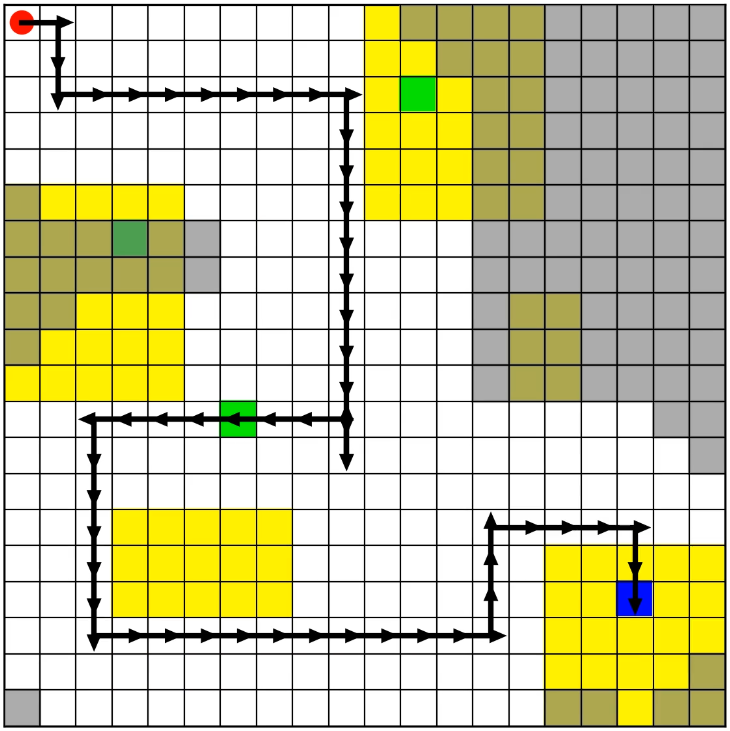}
            \caption{}
        \end{subfigure}
        
    \end{tabular}
    \caption{\scriptsize Trajectories generated by our algorithm that satisfy $\phi_1$ under different maps.}
\label{fig:case1_3}    
\end{figure}

Figure~\ref{fig:case1_3} presents the final traversed trajectories for the same specification under different map configurations by using our algorithm. In Fig.~\ref{fig:case1_3}(a), during the early stage of exploration, the robot discovers a safe exit located near the person in the lower-level area. It immediately terminates the frontier-based exploration and follows the path to satisfy the task. In Fig.~\ref{fig:case1_3}(b), the robot first identifies a safe exit and later encounters the person. When this happens, it stops the frontier-based exploration and follows a trajectory that first reaches the person and then proceeds to the safe exit. Finally, in Fig.~\ref{fig:case1_3} (c), the robot encounters the person midway through the mission and moves toward them, since this progress carries no future consequences (i.e., it does not correspond to a commit state). The robot then rescues the person by transporting them to the first safe exit it locates. Overall, as long as there exists a feasible path to satisfy the specification, our algorithm eventually guarantees satisfaction (Theorem~\ref{theo}), whereas other existing methods may fail in such scenarios, as in Fig.~\ref{fig:case1_2}.

\subsection{Benchmark analysis}

We compare our method with the approach proposed in \cite{ayala2013temporal} in terms of satisfaction rate and average trajectory length under the same task $\phi_1$. For map generation, we conduct a Monte Carlo simulation by generating 500 random grid maps of size $20 \times 20$. Labels are assigned by placing $\bm{n}$ randomly positioned blocks of size $5 \times 5$, each designated with the label $\textbf{L}$. In addition, two cells are labeled as $\textbf{P}$ and two cells as $\textbf{S}$. To satisfy Assumption~2, we ensure that at least one $\textbf{P}$ cell and one $\textbf{S}$ cell are located outside the $\textbf{L}$ blocks and remain reachable without traversing the $\textbf{L}$ regions.

As shown in Table~\ref{tab:table}, when we set the value of $\bm{n} = 0$, there are no $\textbf{L}$ regions in the environment. Therefore, product automaton states that have a commit state as their DFA state are not reachable. This follows from the specification, since the only way to reach a commit state is by observing $\textbf{L}$. In this case, both our algorithm and the method in \cite{ayala2013temporal} satisfy the specification in all 500 maps, which results in a $100\%$ satisfaction rate. In terms of average trajectory length, our algorithm with $46.20$ steps outperforms \cite{ayala2013temporal} with $56.47$ steps. This is because our frontier-based exploration incorporates information gain, which enables more efficient exploration and faster discovery of the desired labels. When we set $\bm{n} = 5$, our algorithm again outperforms \cite{ayala2013temporal}. It achieves a $100\%$ satisfaction rate, whereas their method achieves only $35\%$. This performance gap arises because their approach lacks a mechanism to prevent entering a commit state, which causes the robot to move into an $\textbf{L}$ region and become trapped without a reachable $\textbf{S}$ and $\textbf{P}$. Note that the lower average trajectory length reported for \cite{ayala2013temporal} occurs because, once the robot becomes trapped in an $\textbf{L}$ region with no safe frontier, the algorithm is terminated and the trajectory length is recorded. Thus, their average trajectory length reflects both successful and deadlock cases.

\begin{table}[!htbp]
    \centering
    \resizebox{0.85\columnwidth}{!}{%
    \setlength{\tabcolsep}{4pt} 
    \renewcommand{\arraystretch}{1.1} 
    \begin{tabular}{|c|c|c|c|}
        \hline
        $\bm{n}$ & \textbf{Method} & \textbf{Avg. Trajectory Length} & \textbf{Satisfaction Rate} \\ 
        \hline
        \multirow{2}{*}{$0$} & Our Algorithm & 46.20 & 100.0 \% \\ 
        \cline{2-4}
                               & \cite{ayala2013temporal} & 56.47 & 100.0 \% \\
        \hline
        \multirow{2}{*}{$5$} & Our Algorithm & 48.07 & 100.0 \% \\ 
        \cline{2-4}
                                         & \cite{ayala2013temporal} & 31.43 & 35.00 \% \\
        \hline
    \end{tabular}%
    } 
    \caption{\scriptsize Comparison between \cite{ayala2013temporal} and our proposed algorithm with respect to average trajectory length and success rate, evaluated over 500 maps.}
    \label{tab:table}
\end{table}
\section{Conclusion}\label{sec:conclusion}

In this paper, we present a framework for motion planning in unknown environments under syntactically co-safe Linear Temporal Logic (scLTL) tasks. The proposed algorithm encourages progress toward task satisfaction while avoiding commit states, which are states that restrict satisfaction to a limited set of paths. Our method combines frontier-based exploration with an automata-theoretic approach, which enables the selection of frontiers that balance information gain, travel distance, and progress toward task satisfaction. We demonstrate that if a satisfying path exists, our algorithm eventually finds it, and otherwise it correctly determines that the task is unsatisfiable. Through direct comparison with a state-of-the-art method, we show that our approach achieves better performance and succeeds in cases where the baseline method fails. For future work, we focus on extending this framework to continuous spaces and exploring collaborative strategies involving multiple robots.

\bibliographystyle{IEEEtran}
\bibliography{references}

\end{document}